\documentclass{article}

\usepackage{microtype}
\usepackage{graphicx}
\usepackage{booktabs} % for professional tables

\usepackage{amsmath}
\usepackage{amsfonts}
\usepackage{multirow}
\usepackage{makecell}
\usepackage{rotating}

\usepackage{subcaption}
\usepackage[font=bf,labelfont=bf]{caption}

\usepackage{hyperref}

\usepackage{amsthm}
\newtheorem{thm}{Theorem}

\newtheorem{lem}{Lemma}

\newcommand{\x}{\mathbf{x}}
\newcommand{\X}{\mathbf{X}}

\newcommand{\nphi}{n_{\mathbf{\phi}}}
\newcommand{\mii}{\mathbb{I}}
\newcommand{\miiphi}{\mathbb{I}_{\mathbf{\phi}}}
\newcommand{\miin}{\widehat{\mathbb{I}}_{n}}
\newcommand{\ejxy}{\mathbb{E}_{(\mathbf{X},Y)}}
\newcommand{\emxy}{\mathbb{E}_{\mathbf{X}}\mathbb{E}_{Y}}
\newcommand{\ejxyn}{\mathbb{E}_{(\mathbf{X},Y)_n}}

\newcommand{\ex}{\mathbb{E}_{\mathbf{X}}}
\newcommand{\ey}{\mathbb{E}_{Y}}
\newcommand{\exn}{\mathbb{E}_{\mathbf{X}_n}}
\newcommand{\eyn}{\mathbb{E}_{Y_n}}

\DeclareMathOperator{\PMI}{PMI}

\newcommand{\etal}{\textit{et al}. }
\newcommand{\ie}{\textit{i}.\textit{e}., }

\usepackage[accepted]{icml2021}

\icmltitlerunning{Neural Network Classifier as Mutual Information Evaluator}

\begin{document}

\twocolumn[
\icmltitle{
Neural Network Classifier as Mutual Information Evaluator
}

\icmlsetsymbol{equal}{*}

\begin{icmlauthorlist}
\icmlauthor{Zhenyue Qin}{equal,anu}
\icmlauthor{Dongwoo Kim}{equal,anu,postech}
\icmlauthor{Tom Gedeon}{anu}
\end{icmlauthorlist}

\icmlaffiliation{anu}{School of Computing, Australian National University}
\icmlaffiliation{postech}{GSAI, POSTECH, Korea}

\icmlcorrespondingauthor{Zhenyue Qin}{zhenyue.qin@anu.edu.au}
\icmlcorrespondingauthor{Dongwoo Kim}{dongwoo.kim@anu.edu.au}


\vskip 0.3in
]




% Affiliations 
\printAffiliationsAndNotice{\icmlEqualContribution} % otherwise use the standard text.

\begin{abstract}
Cross-entropy loss with softmax output is a standard choice to train neural network classifiers. 
We give a new view of neural network classifiers with softmax and cross-entropy as mutual information evaluators. 
We show that when the dataset is balanced, training a neural network with cross-entropy maximises the mutual information between inputs and labels through a variational form of mutual information. Thereby, we develop a new form of softmax that also converts a classifier to a mutual information evaluator when the dataset is imbalanced. 
Experimental results show that the new form leads to better classification accuracy, in particular for imbalanced datasets. 
\end{abstract}

\section{Introduction}
Neural network classifiers play an important role in contemporary machine learning and computer vision~\cite{lecun2015deep}. Although many architectural choices and optimisation methods have been explored, relatively fewer considerations have been shown on the final layer of the classifier: the cross-entropy loss with the softmax output. 

The combination of softmax with cross-entropy is a standard choice to train neural network classifiers. It measures the cross-entropy between the ground truth label $y$ and the output of the neural network $\hat{y}$. The network's parameters are then adjusted to reduce the cross-entropy via back-propagation. While it seems sensible to reduce the cross-entropy between the labels and predicted probabilities, it still remains a question as to what relation the network aims to model between input $x$ and label $y$ via this loss function, \ie, softmax with cross-entropy. 

In this work, for neural network classifiers, we explorer the connection between \emph{cross-entropy with softmax} and \emph{mutual information between inputs and labels}. From a variational form of mutual information, we prove that optimising model parameters using the softmax with cross-entropy is equal to maximising the mutual information between input data and labels when the distribution over labels is uniform. This connection provides an alternative view on neural network classifiers: they are mutual information estimators. We further propose a probability-corrected version of softmax that relaxes the uniform distribution condition. We empirically demonstrate that our mutual information estimators can \emph{accurately} evaluate mutual information. We also show mutual information estimators can perform classification more accurately than traditional neural network classifiers. When the dataset is imbalanced, the estimators outperform the state-of-the-art classifier for our example.

\section{Preliminaries}
\label{sec:pre}
In this section, we first define the notations used throughout this paper. We then introduce the definition of mutual information and variational forms of mutual information. 

\subsection{Notation}
We let training data consist of $M$ classes and $N$ labelled instances as $\{ (\mathbf{x}_{i}, y_{i}) \}_{i=1}^{N}$, where $y_i \in \mathcal{Y} = \{ 1, ... , M \}$ %\remove{of each datum $(\mathbf{x}_n, y_n)$} 
is a class label of input $\mathbf{x}_i$. We let $n_{\mathbf{\phi}}(\mathbf{x}): \mathcal{X} \rightarrow \mathbb{R}^M$ be a neural network parameterised by $\phi$, where $\mathcal{X}$ is a space of input $\mathbf{x}$. 
% Without additional clarification, 
We assume $\mathcal{X}$ to be a compact subset of $D$-dimensional Euclidean space. 
We denote by $P_{XY}$ some joint distribution over $\mathcal{X} \times \mathcal{Y}$, with $(\mathbf{X}, Y) \sim P_{XY}$ being a pair of random variables. $P_{X}$ and $P_{Y}$ are the marginal distributions of $\mathbf{X}$ and $Y$, respectively. We remove a subscript from the distribution if it is clear from context.

\subsection{Variational Bounds of Mutual Information}
Mutual information evaluates the mutual dependence between two random variables. 
The mutual information between $\mathbf{X}$ and $Y$ can be expressed as:
% \vspace{-1.mm}
\begin{gather}
\label{eq:orig_mi_def}
\mathbb{I}(\mathbf{X}, Y) = \notag \\ 
\int_{\mathbf{x} \in \mathcal{X}} \bigg[ \sum_{y \in \mathcal{Y}} P(\mathbf{x}, y) \log \big( \frac{P(\mathbf{x}, y)}{P(\mathbf{x}) P(y)} \big) \bigg] d\mathbf{x}. 
\end{gather}

Equivalently, following~\cite{poole2019variational}, we may express the definition of mutual information in \autoref{eq:orig_mi_def} as:
\begin{equation}
\label{eq:cond_mi_def}
\mathbb{I}(\mathbf{X}, Y) = \mathbb{E}_{(\mathbf{X}, Y)} \bigg[ \log \frac{P(y | \mathbf{x})}{P(y)} \bigg],
\end{equation}
where $\mathbb{E}_{(\X, Y)}$ is the abbreviations of $\mathbb{E}_{(\X, Y) \sim P_{XY}}$. 
Computing mutual information directly from the definition is, in general, intractable due to integration. %We therefore derive its variational lower bound to make it tractable. 

%\subsubsection{Barber-Agakov (BA) Representation}
% \label{subsec:ba_rep}
\textbf{Variational form}: Barber and Agakov introduce a commonly used lower bound of mutual information via a variational distribution $Q$~\cite{barber2003algorithm}, derived as the following form:
\begin{align}
\label{eq:ba_def}
\mathbb{I}(\mathbf{X}, Y) &= \mathbb{E}_{(\X, Y)} \bigg[ \log \frac{P(y | \mathbf{x})}{P(y)} \bigg] \notag \\
&= \mathbb{E}_{(\X, Y)} \bigg[ \log \frac{Q(y | \mathbf{x})}{P(y)} \frac{P(y | \mathbf{x})}{Q(y | \mathbf{x})} \bigg] \notag \\
&= \mathbb{E}_{(\X, Y)} \bigg[ \log \frac{Q(y | \mathbf{x})}{P(y)}
\bigg] \notag \\ 
&\ + 
\underbrace{\mathbb{E}_{(\X, Y)} \bigg[ \log \frac{P(\mathbf{x},y)}{Q(\mathbf{x},y)} \bigg]}_{D_{KL}(P(\mathbf{x}, y) || Q(\mathbf{x}, y))} -   
\underbrace{\mathbb{E}_{(\X)} \bigg[ \log \frac{P(\mathbf{x})}{Q(\mathbf{x})} \bigg]}_{D_{KL}(P(\mathbf{x}) || Q(\mathbf{x}))} \notag \\
&\ge \mathbb{E}_{(\X, Y)} \bigg[ \log \frac{Q(\mathbf{x},y)}{P(\x)P(y)} \bigg]. 
\end{align}
The inequality in \autoref{eq:ba_def} holds since KL divergence maintains non-negativity. This lower bound is tight when variational distribution $Q(\mathbf{x},y)$ converges to joint distribution $P(\mathbf{x},y)$, i.e., $Q(\mathbf{x},y) = P(\mathbf{x},y)$.

The form in \autoref{eq:ba_def} is, however, still hard to compute since it is not easy to make a tractable and flexible variational distribution $Q(\mathbf{x},y)$. Variational distribution $Q(\mathbf{x},y)$ can be considered as a constrained function which has to satisfy the probability axioms. Especially, the constraint is challenging to model with a function estimator such as a neural network. To relax the function constraint, McAllester \etal \cite{mcallester2018formal} further apply reparameterisation and define $Q(\mathbf{x},y)$ in terms of an unconstrained function $f_{\phi}$ parameterised by $\phi$ as:
\begin{equation}
\label{eq:repara_q}
Q(\mathbf{x},y) = \frac{P(\x)P(y)}{E_{y' \sim P_Y}[ \exp(f_{\mathbf{\phi}}(\mathbf{x}, y')) ]} \exp(f_{\mathbf{\phi}}(\mathbf{x}, y)).
\end{equation}
As a consequence, the variational lower bound of mutual information $\mathbb{I}(\mathbf{X}, Y)$ can be rewritten with function $f_{\mathbf{\phi}}$ as:
\begin{align}
\label{eq:mi_f_low_bound}
\mathbb{I}(\X, Y) \ge \mathbb{E}_{(\X, Y)} \bigg[ \log \frac{\exp(f_{\mathbf{\phi}}(\mathbf{x}, y))}{ E_{y'}[ \exp(f_{\mathbf{\phi}}(\mathbf{x}, y')) ]}  \bigg].
\end{align}
Thus, one can estimate mutual information without any constraint on $f$.
%where $f_{\mathbf{\phi}}$ is modelled by any neural network thanks to the removal of the function constraint~\cite{mcallester2018formal}. 
Through the reparameterisation, the MI estimation can be recast as an optimisation problem.

\section{NN Classifiers as MI Estimators}
\label{sec:nn_as_mi_evaluators}
In this section, we prove that a neural network classifier with cross entropy loss and softmax output estimates the mutual information between inputs and labels.

To view neural network classifiers as mutual information estimators, we need to discuss two separate cases related to the dataset: whether it is balanced or imbalanced. 

\subsection{Softmax with Balanced Dataset}
Softmax is widely used to map outputs of neural networks into a categorical probabilistic distribution for classification. Given neural network $n(\x):\mathcal{X} \rightarrow \mathbb{R}^{M}$, softmax $\sigma:\mathbb{R}^{M} \rightarrow \mathbb{R}^{M}$ is defined as:
\begin{align}
\sigma(n(\x))_y = \frac{\exp( n(\x)_y )}{\sum_{y'=1}^{M}\exp( n(\x)_{y'})}.
\end{align}
Expected cross-entropy is often employed to train a neural network with softmax output. The expected cross-entropy loss is
\begin{align}
\label{eq:cross_entropy}
L = - \mathbb{E}_{(\X,Y)}[ n(\x)_y - \log({\sum_{y'=1}^{M}\exp( n(\x)_{y'})}) ],
\end{align}
where the expectation is taken over the joint distribution of $X$ and $Y$. Given a training set, one can train the model with an empirical distribution of the joint distribution. We present an interesting connection between cross-entropy with softmax and mutual information in the following theorem. In a bid for conciseness, we only provide proof sketches for \autoref{thm:equality} and \autoref{thm:softmax_im} here. Please refer to the appendix for rigorous proofs. 
\begin{thm}
 \label{thm:equality}
 Let $f_\phi(\x,y)$ be $n(\x)_y$. Infimum of the expected cross-entropy loss with softmax outputs is equivalent to the mutual information between input and output variables up to constant $\log M$ under uniform label distribution. %  maximising \autoref{eq:mi_f_low_bound}, \ie the lower bound of mutual information, under the uniform label distribution. That is, if the dataset is balanced, then training a neural network via minimising cross-entropy with softmax equals enhancing a estimator toward more accurately evaluating the mutual information between data and label. 
\end{thm}

% \begin{proof}
% Let $f_\phi(\x,y) = n(\x)_y$, then the lower bound is
% \begin{align}
%  \mathbb{E}_{(\X, Y)} \bigg[ \log \frac{\exp(n(\x)_y)}{ E_{y'}[ \exp(n(\x)_{y'}) ]}  \bigg].
% \end{align}
% If the distribution of the label is uniform then, it can be rewritten as
% \begin{align}
%  &\mathbb{E}_{(\X, Y)} \bigg[ \log \frac{\exp(n(\x)_y)}{ 1/M \sum_{y'=1}^{M} \exp(n(\x)_{y'}) }  \bigg] \notag \\
%  &= \mathbb{E}_{(\X, Y)} \bigg[ \log \frac{\exp(n(\x)_y)}{ \sum_{y'=1}^{M} \exp(n(\x)_{y'}) } \bigg] + \log M, \label{eq:softmax_mi}
% \end{align}
% which is equivalent to the negative expected cross-entropy loss (\ref{eq:cross_entropy}) up to constant $\log M$. Hence, the infimum of the expected cross entropy is equal to the mutual information between input and output variables since the supremum of r.h.s in \autoref{eq:mi_f_low_bound} is the mutual information. %Hence, by minimising the cross-entropy, we can obtain the lower bound of mutual information.
% \end{proof}
\begin{proof}
See Appendix. 
\end{proof}
Note that the constant does not change the gradient of the objective. Consequently, the solutions of both the mutual information maximisation and the softmax cross-entropy minimisation optimisation problems are the same. 

\subsection{Softmax with Imbalanced Dataset}
% Please check (online) the definition of Imbalanced and Unbalanced - I think you mean Unbalanced
The uniform label distribution assumption in \autoref{thm:equality} is restrictive since we cannot access the true label distribution, often assumed to be non-uniform. To relax the restriction, we propose a probability-corrected softmax (PC-softmax):
\begin{align}
\label{eq:prob_cor_softmax}
\sigma_p(n(\x))_y = \frac{\exp( n(\x)_y )}{\sum_{y'=1}^{M}P(y')\exp( n(\x)_{y'})},
\end{align}
where $P(y')$ is a distribution over label $y'$. 
% For the sake completeness, even though sigmoid is a special case of softmax, we include the pc-sigmoid as well: 
% \begin{align}
% \label{eq:pc_sigmoid}
% \sigma_s(n(\x))_y = \frac{1 / P(y_1)}{1 + (P(y_0) / P(y_1)) \exp( n(\x)_{y_1})},
% \end{align}
% where $y_1$ and $y_0$ stand for belonging to and not belonging to the class, respectively. 
In experiments, we optimise the revised softmax with empirical distribution on ${P}(y')$ estimated from the training set. We show the equivalence between optimising the classifier and maximising mutual information with the new softmax below.

\begin{thm}
\label{thm:softmax_im}
The mutual information between two random variables $X$ and $Y$ can be obtained via the infimum of cross-entropy with PC-softmax in \autoref{eq:prob_cor_softmax}, using a neural network. Such an evaluation is strongly consistent. % Not sure what this last sentence means
\end{thm}
% See \autoref{sec:proofs} for the proof of \autoref{thm:softmax_im}.
See the proofs in the appendix for the proof of \autoref{thm:softmax_im}.
% \begin{proof}
%It has been shown that the equality in \autoref{eq:ba_def} holds iff r.h.s of \autoref{eq:mi_f_low_bound} is in its supremum~\cite{mcallester2018formal}. 
%First, it can be easily shown that we can relax the uniform assumption with PC-softmax. We then show that the class of functions modelled by $n:\mathcal{X} \rightarrow \mathbb{R}^{M}$ is the same as those of $f:\mathcal{X}\times\mathcal{Y} \rightarrow \mathbb{R}$.
% $Y$ is a categorical variable. Hence, unconstrained function $f$ can be decomposed into a set of functions indexed by $y$, i.e., $f = \{f_{y} \}_{y=1}^{M}$. Under a mild condition, $n(x)_y$ can approximate any continuous function by the universal approximation theorem~\cite{hornik1989multilayer}. We conclude the sketch proof by letting $n(x)_y$ be $f_{y}$ so that $n(x)_y$ is the mutual information evaluator.
% \end{proof}

Mutual information is often used in generative models to find the maximally informative representation of an observation~\cite{hjelm2018learning,zhao2017infovae}, whereas its implication in classification has been unclear so far. The results of this section imply that the neural network classifier with softmax optimises its weights to maximise the mutual information between inputs and labels under the uniform label assumption. 
% We further study an application of this implication in \autoref{sec:cam} to tackle the weakly supervised object localisation task.
%In \autoref{sec:mi_class_exp}, we show the empirical difference between softmax and PC-softmax via synthetic and real world datasets on the mutual information estimation and classification tasks.

\begin{table}[t!]
    \centering
    \begin{tabular}{c r r r}
        \toprule
        $y$ & $\mu$ & \# samples & $p(y)$\\
        \midrule 
        0 & $\mathbf{0}$ & 6,000 & 0.07 \\
        1 & $+\mathbf{2}$ & 12,000 & 0.13 \\
        2 & $-\mathbf{2}$ & 18,000 & 0.20 \\
        3 & $+\mathbf{4}$ & 24,000 & 0.27 \\
        4 & $-\mathbf{4}$ & 30,000 & 0.33 \\
        \bottomrule
    \end{tabular}
    \vspace{1em}
    \caption{Synthetic dataset description. $\mu$ is a mean vector for each Gaussian distribution. \# samples denotes the number (resp. prior distribution) of samples with the non-uniform prior assumption. For the test with the uniform prior assumption, we use 12,000 samples from each distribution.}
    \label{table:synthetic_dataset_spec}
\end{table}

\begin{table}[t!]
\begin{subtable}[t]{\linewidth}
\centering
\resizebox{0.75\textwidth}{!}{
\begin{tabular}{r r r r r}
\toprule
\multirow{2}{*}{Dimension}&
\multirow{2}{*}{Accuracy(\%)}& 
\multicolumn{3}{c}{Mutual information}\\
\cmidrule(lr{1em}){3-5}
& & MC & MINE & softmax \\
\midrule
1 & 74 & 1.03 & 1.00 & 0.99\\
2 & 85 & 1.30 & 1.22 & 1.28\\
5 & 94 & 1.54 & 1.46 & 1.48\\
10 & 98 & 1.60 & 1.54 & 1.54\\
\bottomrule
\end{tabular}
}
\caption{Results with balanced datasets.}
\label{table:syn_balanced}
\end{subtable}

\vspace{1em}

\begin{subtable}[t]{\linewidth}
\centering
\resizebox{0.99\textwidth}{!}{
\begin{tabular}{r r r r r r r}
\toprule
\multirow{2}{*}{Dimension}&
\multicolumn{2}{c}{Accuracy(\%)}&
\multicolumn{4}{c}{Mutual information}\\
\cmidrule(lr{1em}){2-3}
\cmidrule(lr{1em}){4-7}
& softmax & PC-softmax & MC & MINE & softmax & PC-softmax \\
\midrule
1 & 79 & 79 & 1.02 & 0.99 & 1.11 & 0.96 \\
2 & 87 & 88 & 1.23 & 1.17 & 1.31 & 1.20 \\
5 & 93 & 95 & 1.44 & 1.27 & 1.41 & 1.31 \\
10 & 95 & 96 & 1.48 & 1.22 & 1.36 & 1.34 \\
\bottomrule
\end{tabular}
}
\caption{Results with unbalanced datasets.} % Acc. stands for the classification accuracy with Softmax and PC-Softmax, respectively. } % Acc. is also defined in the caption below
\label{table:syn_imbalanced}
\end{subtable}

\caption{Mutual information estimation results with softmax-based classification neural networks. MC represents the estimated mutual information via Monte Carlo methods.}
\label{table:cam_info_cam}
\end{table}

\section{Impact of PC-softmax on Classification}
\label{sec:mi_class_exp}
%In the previous section, we show that classification neural networks can be utilised to measure the mutual information (MI) between continuous and discrete distributions. 
In this section, we measure the empirical performance of PC-softmax as mutual information (MI) and the influence of PC-softmax on the classification task. Since it is impossible to obtain correct MI from real-world datasets, we first construct synthetic data with known properties to measure the MI estimation performance, and then we use two real-world datasets to measure the impact of PC-softmax on classification tasks.
%We display both softmax and the new probability-corrected (PC) softmax outlined in \autoref{eq:prob_cor_softmax} can approximate MI even under significantly imbalanced dataset, despite the correctness of the corrected approach.

\subsection{Mutual information estimation task}

To construct a synthetic dataset with a pair of continuous and discrete variables, we employ a Gaussian mixture model:
\begin{align}
    P(x) &= \sum_{y=1}^{M} P(y) \mathcal{N}(\mathbf{x} | \mathbf{\mu}_y, \mathbf{\Sigma}_y) \notag\\
    P(x | y) &= \mathcal{N}(\mathbf{x} | \mathbf{\mu}_y, \mathbf{\Sigma}_y), \notag
\end{align}
where $P(y)$ is a prior distribution over the labels. To form classification, we use $x$ as an input variable, and $y$ as a label.

For the experiments, we use five mixtures of isotropic Gaussian, each of which has a unit diagonal covariance matrix with different means. We set the parameters of the mixtures to make them overlap in significant proportions.

% Separate discussion on balance
We generate two sets of datasets: one with uniform prior and the other with non-uniform prior distribution over labels, $p(y)$. For the uniform prior, we sample 12,000 data points from each Gaussian, and for the non-uniform prior, we sample unequal number of data points from each Gaussian. In addition, we vary the dimension of Gaussian distribution from 1 to 10. The detailed statistics for the Gaussian parameters and the number of samples are available in \autoref{table:synthetic_dataset_spec}. To train classification models, we divide the dataset into training, validation and test sets. We use the validation set to find the best parameter configuration of the classifier.

We aim to compare the difference of true and softmax-based estimated mutual information $\mathbb{I}(\mathbf{X}, Y)$. 
The mutual information is, however, intractable. We thus approximate it via Monte Carlo (MC) methods using the true probability density function, expressed as:
\begin{equation}
    \label{eq:mc_mi}
    \mathbb{I}(\X, Y) \approx \frac{1}{N} \sum_{i=1}^{N} \log \left( \frac{P(\mathbf{x}_i | y_i)}{P(\mathbf{x}_i)} \right), 
\end{equation}
where $(\mathbf{x}_i, y_i)$ forms a paired sample. \autoref{eq:mc_mi} attains equality as $N$ approaches infinity. 

%To estimate mutual information via classification, we train classification models with two versions of softmax.

We use four layers of a feed-forward neural network with the ReLU as an activation for internal layers and softmax as an output layer
% \footnote{All model details used in this paper are available in the supplementary material.}. 
We train the model with softmax on balanced dataset and with PC-softmax on unbalanced dataset. We compare the experimental results against mutual information neural estimator (MINE) proposed in \cite{belghazi2018mutual}. Note that MINE requires having a pair of input and label variables as an input of an estimator network, the classification-based MI-estimator seems more straightforward for measuring mutual information between inputs and labels of classification tasks.

\autoref{table:syn_balanced} summarises the experimental results with the balanced dataset. With the balanced dataset, there is no difference between softmax and PC-softmax. Note that the MC estimator has access to explicit model parameters for estimating mutual information, whereas the softmax estimator measures mutual information based on the model outputs without accessing the true distribution. We could not find a significant difference between MC and the softmax estimator. Additionally, we report the accuracy of the trained model on the classification task.

\autoref{table:syn_imbalanced} summarises the experimental results with the unbalanced dataset. The results show that the PC-softmax slightly under-estimates mutual information when compared with the other two approaches. It is worth noting that the classification accuracy of PC-softmax consistently outperforms the original softmax. 
The results show that the MINE slightly under-estimate the MI as the input dimension increases.
%However, MI evaluator based on PC-softmax does not require the marginal distributions of inputs and labels, which may potentially reduce the complexity of model architecture.
%is less intensive since ours does not require taking both joint and marginal %word missing here: joint and marginal <something> of data ...
%of data and labels as inputs. 
% although it is not statistically significant.

% there will be a classification result on the real world dataset mnist and cub at the end of this section.

\subsection{Classification task}
%In this section, we demonstrate that maximising mutual information can result in classifiers with higher classification accuracy. 

We test the classification performance of softmax and PC-softmax with two real-world datasets: MNIST~\cite{lecun2010mnist} and CUB-200-2011~\cite{wah2011caltech}. 
%We utilise the former due to its simplicity and classic for evaluating a classifier. However, MNIST can be overly simple for assessing modern classifiers. Thus, we apply a more challenging dataset that is still popular for testing the current classifiers.

%To be more specific with the dataset settings, w
We construct balanced and unbalanced versions of the MNIST dataset.
For the balanced-MNIST, we use a subset of the original dataset. For the unbalanced-MNIST, we randomly subsample one tenth of instances for digits 0, 2, 4, 6 and 8 from the balanced-MNIST. 
With CUB-200-2011, we follow the same training and validation splits as in~\cite{cui2018large}. As a result of such splitting, the training set is approximately balanced, where out of the total 200 classes, 196 of them contain 30 instances and the remaining 6 classes include 29 instances. To construct an unbalanced dataset, similar to MNIST, we randomly drop one half of the instances from one half of the bird classes.

We adopt a simple convolutional neural network as a classifier for MNIST. The model contains two convolutional layers with max pooling layer and the ReLU activation, followed by two fully connected layers with the final softmax. For CUB-200-2011, we apply the same architecture as Inception-V3~\cite{cui2018large}.%, which demonstrates the state-of-the-art classification performance in CUB-200-2011 after being fine-tuned~\cite{cui2018large}. 
We measure both the micro accuracy and the average per-class accuracy of the two softmax versions on both datasets. The average per-class accuracy alleviates the dominance of the majority classes in unbalanced datasets. The classification results are shown in \autoref{table:emp_soft}. PC-softmax is significantly more accurate on unbalanced datasets for the average per-class accuracy. %supported by the Mann-Whitley statistical test results
%, while the results are similar for both softmax versions on standard classification accuracy. 

%however, the difference is not statistically significant. %There is no significant difference in performance between two softmax.

\begin{table}[t!]
\begin{subtable}[t]{\linewidth}
\centering
\resizebox{0.99\textwidth}{!}{
\begin{tabular}{c r r r r r}
\toprule
\multirow{2}{*}{Dataset} & \multicolumn{2}{c}{MNIST} & \multicolumn{2}{c}{CUB-200-2011} \\
&
Balanced& 
Unbalanced& 
Balanced& 
Unbalanced\\
\midrule
softmax & 97.95 & 96.81 & 89.23 & 89.21  \\
PC-softmax & 97.91 & 96.86 & 89.18 & \textbf{89.73}* \\
%p-value & 0.28& 0.01& 0.40& 0.01 \\
\bottomrule
\end{tabular}
}
\caption{Classification accuracy (\%).}

\end{subtable}

\vspace{1em}

\begin{subtable}[t]{\linewidth}
\centering
\resizebox{0.99\textwidth}{!}{
\begin{tabular}{c r r r r r}
\toprule
\multirow{2}{*}{Dataset} & \multicolumn{2}{c}{MNIST} & \multicolumn{2}{c}{CUB-200-2011} \\
& 
Balanced& 
Unbalanced& 
Balanaced& 
Unbalanced\\
\midrule
softmax & 97.95 & 95.05 & 89.21 & 84.63 \\
PC-softmax & 97.91 & \textbf{96.30} & 89.16 & \textbf{87.69} \\
%p-value & 0.28& 0.01& 0.83& 0.01 \\

\bottomrule
\end{tabular}
}
\caption{Average per-class accuracy (\%).}
\end{subtable}

\vspace{1em}
\caption{Classification accuracy of using softmax and PC-softmax. Numbers of instances for different labels are the same for a balanced dataset and are significantly distinct for an unbalanced dataset. Bold values denote p-values less than 0.05 with the Mann-Whitney U statistical test\protect\footnotemark.}
\label{table:emp_soft}
\end{table}
% \footnotetext{Accuracy with * is higher than the current state-of-the-art~\cite{cui2018large}.}
\section{Conclusion}

We have shown the connection between mutual information estimators and neural network classifiers through the variational form of mutual information. The connection explains the rationale behind the use of sigmoid, softmax and cross-entropy from an information-theoretic perspective.
% The connection also brings a new insight to understand neural network classifiers. 
There is previous work that called the negative log-likelihood (NLL) loss as maximum mutual information estimation~\cite{bahl1986maximum}. Despite this naming similarity, that work does not show the relationship between softmax and mutual information that we have shown here. 
The connection between neural network classifiers and mutual information evaluators provides more than an alternative view on neural network classifiers. Thereby, we improve the classification accuracy, in particular when the datasets are unbalanced. The new mutual information estimators even outperform the prior state-of-the-art neural network classifiers.

\newpage
\nocite{langley00}

\bibliography{refs}

\begin{thebibliography}{13}
\providecommand{\natexlab}[1]{#1}
\providecommand{\url}[1]{\texttt{#1}}
\expandafter\ifx\csname urlstyle\endcsname\relax
  \providecommand{\doi}[1]{doi: #1}\else
  \providecommand{\doi}{doi: \begingroup \urlstyle{rm}\Url}\fi

\bibitem[Bahl et~al.(1986)Bahl, Brown, De~Souza, and Mercer]{bahl1986maximum}
Bahl, L.~R., Brown, P.~F., De~Souza, P.~V., and Mercer, R.~L.
\newblock Maximum mutual information estimation of hidden markov model
  parameters for speech recognition.
\newblock In \emph{Proc. ICASSP}, volume~86, pp.\  49--52, 1986.

\bibitem[Barber \& Agakov(2003)Barber and Agakov]{barber2003algorithm}
Barber, D. and Agakov, F.~V.
\newblock The im algorithm: a variational approach to information maximization.
\newblock In \emph{Advances in Neural Information Processing Systems}, pp.\
  None, 2003.

\bibitem[Belghazi et~al.(2018)Belghazi, Baratin, Rajeshwar, Ozair, Bengio,
  Courville, and Hjelm]{belghazi2018mutual}
Belghazi, M.~I., Baratin, A., Rajeshwar, S., Ozair, S., Bengio, Y., Courville,
  A., and Hjelm, D.
\newblock Mutual information neural estimation.
\newblock In \emph{International Conference on Machine Learning}, pp.\
  531--540, 2018.

\bibitem[Cui et~al.(2018)Cui, Song, Sun, Howard, and Belongie]{cui2018large}
Cui, Y., Song, Y., Sun, C., Howard, A., and Belongie, S.
\newblock Large scale fine-grained categorization and domain-specific transfer
  learning.
\newblock In \emph{Proceedings of the IEEE Conference on Computer Vision and
  Pattern Recognition}, pp.\  4109--4118, 2018.

\bibitem[Geer \& van~de Geer(2000)Geer and van~de Geer]{geer2000empirical}
Geer, S.~A. and van~de Geer, S.
\newblock \emph{Empirical Processes in M-estimation}, volume~6.
\newblock Cambridge University Press, 2000.

\bibitem[Hjelm et~al.(2019)Hjelm, Fedorov, Lavoie-Marchildon, Grewal, Bachman,
  Trischler, and Bengio]{hjelm2018learning}
Hjelm, R.~D., Fedorov, A., Lavoie-Marchildon, S., Grewal, K., Bachman, P.,
  Trischler, A., and Bengio, Y.
\newblock Learning deep representations by mutual information estimation and
  maximization.
\newblock In \emph{International Conference on Learning Representation}, 2019.

\bibitem[Hornik et~al.(1989)Hornik, Stinchcombe, and
  White]{hornik1989multilayer}
Hornik, K., Stinchcombe, M., and White, H.
\newblock Multilayer feedforward networks are universal approximators.
\newblock \emph{Neural Networks}, 2\penalty0 (5):\penalty0 359--366, 1989.

\bibitem[LeCun et~al.(2010)LeCun, Cortes, and Burges]{lecun2010mnist}
LeCun, Y., Cortes, C., and Burges, C.
\newblock Mnist handwritten digit database.
\newblock 2010.

\bibitem[LeCun et~al.(2015)LeCun, Bengio, and Hinton]{lecun2015deep}
LeCun, Y., Bengio, Y., and Hinton, G.
\newblock Deep learning.
\newblock \emph{Nature}, 521\penalty0 (7553):\penalty0 436, 2015.

\bibitem[McAllester \& Statos(2018)McAllester and Statos]{mcallester2018formal}
McAllester, D. and Statos, K.
\newblock Formal limitations on the measurement of mutual information.
\newblock \emph{arXiv preprint arXiv:1811.04251}, 2018.

\bibitem[Poole et~al.(2019)Poole, Ozair, Oord, Alemi, and
  Tucker]{poole2019variational}
Poole, B., Ozair, S., Oord, A. v.~d., Alemi, A.~A., and Tucker, G.
\newblock On variational bounds of mutual information.
\newblock In \emph{International Conference on Machine Learning}, 2019.

\bibitem[Wah et~al.(2011)Wah, Branson, Welinder, Perona, and
  Belongie]{wah2011caltech}
Wah, C., Branson, S., Welinder, P., Perona, P., and Belongie, S.
\newblock The caltech-ucsd birds-200-2011 dataset.
\newblock 2011.

\bibitem[Zhao et~al.(2017)Zhao, Song, and Ermon]{zhao2017infovae}
Zhao, S., Song, J., and Ermon, S.
\newblock Infovae: Information maximizing variational autoencoders.
\newblock \emph{arXiv preprint arXiv:1706.02262}, 2017.

\end{thebibliography}
\bibliographystyle{icml2021}

\appendix
\newpage
\onecolumn
\section{Proofs}
\label{sec:proofs}
In this section, we provide rigorous proofs of \autoref{thm:equality} and \autoref{thm:softmax_im}. The structure of proof is similar to the proof used in \cite{belghazi2018mutual}. We assume the input space $\Omega = \X \times Y$ being a compact domain of $\mathcal{R}^{d}$, where all measures are Lebesgue and are absolutely continuous. We restrict neural networks to produce a single continuous output, denoted as $n(\x)_y$. We restate the two theorems for quick reference.

\textbf{Theorem 1.} \textit{Let $f_\phi(\x,y)$ be $n(\x)_y$. Minimising the cross-entropy loss of softmax-normalised neural network outputs is equivalent to maximising \autoref{eq:mi_f_low_bound}, \ie the lower bound of mutual information, under the uniform label distribution. That is, if the dataset is balanced, then training a neural network via minimising cross-entropy with softmax equals enhancing a estimator toward more accurately evaluating the mutual information between data and label.}

\textbf{Theorem 2.} \textit{The mutual information between two random variable $X$ and $Y$ can be obtained via the infimum of cross-entropy with PC-softmax in \autoref{eq:prob_cor_softmax}. Such an evaluation is strongly consistent. }

The proof technique that we have used to prove Theorem 2 is similar to the one used in \cite{belghazi2018mutual}.

\begin{lem}
Let $\eta > 0$. There exists a family of neural network functions $\nphi$ with parameter $\mathbf{\phi}$ in some compact domain such that
\begin{align}
    |\mii(\X;Y) - \miiphi(\X;Y)| \le \eta, 
\end{align}
where 
\begin{align}
    \miiphi(\X;Y) = \underset{\mathbf{\phi}}{\sup} \ \ejxy \big[ \nphi \big] - \ex \log \ey [\exp(\nphi)_y]. 
\end{align}
\end{lem}
\begin{proof}
Let $\nphi^{\ast}(\X,Y) = \PMI(\X,Y) = \log \frac{P(\X, Y)}{P(\X)P(Y)}$. We then have: 
\begin{align}
    \ejxy [\nphi^{\ast}(\x)_y] = \mii(\X;Y) \quad \text{and} \quad \emxy[ \exp{(\nphi^{\ast}(\x)_y)} ] = 1. 
\end{align}
Then, for neural network $\nphi$, the gap $\mii(\X;Y) - \miiphi(\X;Y)$: 
\begin{align}
    \mii(\X;Y) - \miiphi(\X;Y) 
    &= \ejxy [ \nphi^{\ast}(\X,Y) - \nphi(\X,Y) ] + \ex \log \ey [\exp(\nphi)_y] \notag \\
    &\le \ejxy [ \nphi^{\ast}(\X,Y) - \nphi(\X,Y) ] + \log \emxy[\exp(\nphi(\x)_y)] \notag \\
    &\le \ejxy [ \nphi^{\ast}(\X,Y) - \nphi(\X,Y) ] \notag\\
    & \quad + \emxy[\exp(\nphi(\x)_y) - \exp{(\nphi^{\ast}(\x)_y)} ].  \label{eq:mi_diff}
\end{align}
\autoref{eq:mi_diff} is positive since the neural mutual information estimator evaluates a lower bound. The equation uses Jensen's inequality and the inequality $\log x \le x - 1$. 

We assume $\eta > 0$ and consider $\nphi^{\ast}(\x)_y$ is bounded by a positive constant $M$. Via the universal approximation theorem \cite{hornik1989multilayer}, there exists $\nphi(\x)_y \le M$ such that
\begin{align}
    \ejxy | \nphi^{\ast}(\X,Y) - \nphi(\X,Y) | \le \frac{\eta}{2} \quad \text{and} \quad \emxy|\nphi(\x)_y - \nphi^{\ast}(\x)_y | \le \frac{\eta}{2} \exp{(-M)}. 
    \label{eq:mi_diff_up}
\end{align}

By utilising that $\exp$ is Lipschitz continuous with constant $\exp(M)$ over $(-\infty, M]$, we have
\begin{align}
    \emxy|\exp(\nphi(\x)_y) - \exp(\nphi^{\ast}(\x)_y) | \le \exp(M) \cdot  \emxy|\nphi(\x)_y - \nphi^{\ast}(\x)_y | \le \frac{\eta}{2}.
    \label{eq:mi_diff_lips}
\end{align}

Combining \autoref{eq:mi_diff}, \autoref{eq:mi_diff_up} and \autoref{eq:mi_diff_lips}, we then obtain
\begin{align}
    |\mii(\X;Y) - \miiphi(\X;Y)| &\le \ejxy | \nphi^{\ast}(\X,Y) - \nphi(\X,Y) | \notag \\
    &\quad + \emxy|\exp(\nphi(\x)_y) - \exp(\nphi^{\ast}(\x)_y) | \notag \\
    &= \frac{\eta}{2} + \frac{\eta}{2} = \eta. 
\end{align}
\end{proof}

\begin{lem}
Let $\eta > 0$. Given a family of neural networks $\nphi$ with parameter $\mathbf{\phi}$ in some compact domain, there exists $N \in \mathbb{N}$ such that
\begin{align}
    \forall n \ge N, \text{Pr}\big( | \miin(\X;Y) - \nphi(\X;Y) | \le \eta \big) = 1. 
\end{align}
\begin{proof}
We start by employing the triangular inequality: 
\begin{gather}
    | \miin(\X;Y) - \miiphi(\X;Y) | \notag \\
    \le \underset{\mathbf{\phi}}{\sup} \ |\ejxy[\nphi^{\ast}(\X,Y)] - \ejxyn[\nphi^{\ast}(\X,Y)]| \notag \\ 
    + \underset{\mathbf{\phi}}{\sup} \ |\ex \log \ey [\exp(\nphi)_y] - \exn \log \eyn [\exp(\nphi)_y]| \label{eq:mi_diff_n_phi}
\end{gather}

We have stated previously that neural network $\nphi$ is bounded by $M$, \ie $\nphi(\x)_y \le M$. Using the fact that $\log$ is Lipschitz continuous with constant $\exp(M)$ over the interval $[\exp(-M), \exp(M)]$. We have
\begin{equation}
    | \log \ey [\exp(\nphi)_y] - \log \eyn [\exp(\nphi)_y] | \le \exp(M) \cdot |\ey [\exp(\nphi)_y] - \eyn [\exp(\nphi)_y]|
\end{equation}

Using the uniform law of large numbers \cite{geer2000empirical}, we can choose $N \in \mathbb{N}$ such that for $\forall n \ge N$ and with probability one
\begin{equation}
\underset{\mathbf{\phi}}{\sup} \ |\ey [\exp(\nphi)_y] - \eyn [\exp(\nphi)_y]| \le \frac{\eta}{4} \exp(-M). 
\end{equation}
That is, 
\begin{align}
    | \log \ey [\exp(\nphi)_y] - \log \eyn [\exp(\nphi)_y] | \le \frac{\eta}{4}
\end{align}
Therefore, using the triangle inequality we can rewrite \autoref{eq:mi_diff_n_phi} as: 
\begin{gather}
    | \miin(\X;Y) - \miiphi(\X;Y) |
    \le \underset{\mathbf{\phi}}{\sup} \ |\ejxy[\nphi^{\ast}(\X,Y)] - \ejxyn[\nphi^{\ast}(\X,Y)]| \notag \\ 
    + \underset{\mathbf{\phi}}{\sup} \ |\ex \log \ey [\exp(\nphi)_y] - \exn \log \ey [\exp(\nphi)_y]| + \frac{\eta}{4}. 
    \label{eq:mi_diff_add_const}
\end{gather}

Using the uniform law of large numbers again, we can choose $N \in \mathbb{N}$ such that for $\forall n \ge N$ and with probability one
\begin{align}
    \underset{\mathbf{\phi}}{\sup} \ |\ex \log \ey [\exp(\nphi)_y] - \exn \log \ey [\exp(\nphi)_y]| \le \frac{\eta}{4} 
    \label{eq:margin_less}
\end{align}
and: 
\begin{align}
    \underset{\mathbf{\phi}}{\sup} \ |\ejxy[\nphi^{\ast}(\X,Y)] - \ejxyn[\nphi^{\ast}(\X,Y)]| \le \frac{\eta}{2}. 
    \label{eq:join_less}
\end{align}

Combining \autoref{eq:mi_diff_add_const}, \autoref{eq:margin_less} and \autoref{eq:join_less} leads to
\begin{equation}
    | \miin(\X;Y) - \underset{\phi}{\sup} \ \miiphi(\X;Y) | \le \frac{\eta}{2} + \frac{\eta}{4} + \frac{\eta}{4} = \eta. 
\end{equation}
\end{proof}
\end{lem}

Now, combining the above two lemmas, we prove that our mutual information evaluator is strongly consistent. 
\begin{proof}
Using the triangular inequality, we have
\begin{align}
    |\mii(\X;Y) - \miin(\X;Y)| \le |\mii(\X;Y) - \miiphi(\X;Y)| + |\miin(\X;Y) - \nphi(\X;Y)| \le \epsilon. 
\end{align}
\end{proof}

\end{document}